\documentclass[pdflatex,sn-mathphys-num]{sn-jnl}

\usepackage{graphicx}%
\usepackage{multirow}%
\usepackage{amsmath,amssymb,amsfonts}%
\usepackage{amsthm}%
\usepackage{mathrsfs}
\usepackage[title]{appendix}%
\usepackage{xcolor}%
\usepackage{textcomp}%
\usepackage{manyfoot}%
\usepackage{booktabs}%
\usepackage{algorithm}%
\usepackage{algorithmicx}%
\usepackage{algpseudocode}%
\usepackage{listings}%
\usepackage{fontspec}
\usepackage{multirow}
\usepackage{hhline}
\usepackage{array}
\usepackage{pifont}
\usepackage{bm}

\newfontface\mutlufont[
  Path=./
]{Mutlu__Ornamental.ttf}
\newfontface\jsmathfont[
  Path=./
]{jsMath-cmsy10.ttf}
\newcommand{\mutlutext}[1]{{\mutlufont #1}}

\theoremstyle{thmstyleone}%
\newtheorem{theorem}{Theorem}
\theoremstyle{thmstyletwo}%
\theoremstyle{thmstylethree}%
\newtheorem{definition}{Definition}%
\begin{document}

\title[Article Title]{Understanding Deep Learning via Entropy Space Theory}


\author*[1,2,3]{\fnm{Li} \sur{Li}}\email{lil@cust.edu.cn}

\author[1,2,3]{\fnm{Tong} \sur{Zhang}}\email{632177003@qq.com}

\author[1]{\fnm{Wentao} \sur{Yu}}\email{19225012792@163.com}

\author[1,2,3]{\fnm{Zuobin} \sur{Wang}}\email{wangz@cust.edu.cn}

\affil*[1]{\orgdiv{International Reserach Centre for Nano Handling and Manufacturing of China}, \orgname{Changchun University of Science and Technology}, \orgaddress{\street{Weixing Road}, \city{Changchun City}, \postcode{130022}, \state{Jilin Province}, \country{China}}}

\affil[2]{\orgdiv{School of Electronic Information Engineering}, \orgname{Changchun University of Science and Technology}, \orgaddress{\street{Weixing Road}, \city{Changchun City}, \postcode{130022}, \state{Jilin Province}, \country{China}}}

\affil[3]{\orgname{Zhongshan Institute of Changchun University of Science and Technology}, \orgaddress{\street{Huizhan East Road}, \city{Zhongshan City}, \postcode{528437}, \state{Guangdong Province}, \country{China}}}


\abstract{Deep learning is often criticized for its theoretical research lagging behind practice. To make deep learning easier to understand, the entropy space theory is first introduced here. The entropy space can cover all the possibilities of any deep learning model by topological structure. It is independent of network parameters. Through the designed fundamental operations and norm, entropy space is proven to be a normed space within the formal axiomatic framework. Based on the theory, a unified coordinate system is proposed. It can coordinatize every state of a model and rank them by compression of the maximal value of information entropy. The theory offers a novel priori framework for mathematical fundamentals of deep learning.}




\maketitle

\section{Introduction}\label{sec1}
Artificial intelligence (AI) evolves extremely quickly and achieves plenty of practical successes in the past decade\cite{AI_review1,AI_review2,AI_review3,AI_review4,AI_review5}. However, the larger model, the less understandings for human-being self, forms of architectures range from classic machine learnings to deep neural networks and the transformer. Their respective mathematical frameworks usually are reproducing kernel Hilbert space (RKHS)\cite{RKHS}, multi-dimensional vector based nonlinear function\cite{nonlinear_function1,nonlinear_function2} and embedding space\cite{embedding_space}. While pragmatic, these high dimension spaces may be too complex to cognize. Such a ‌complexity hampers human from opening the black box inside AI\cite{blackbox} and from AI being modelled more like human even beyond either\cite{like_human}. Due to the shortcomings of fundamental mathematical and scientific principles\cite{DL_review1}, it is urgently required a theorized paradigm shift which initiates from a novel coordinate system for mutually bridging the gap across human understanding to AI\cite{bridge}.\\
To the best of our knowledge, there are currently only six review articles in deep learning theory field\cite{DL_review2,DL_review1,DL_review3,DL_review4,DL_review5,DL_review6}. The progresses about mathematical framework are starting to take shape in some areas, however still missing the critical piece to understand the whole jigsaw puzzle for various architectures. The geometric deep learning provides a versatile mathematical blueprint expensively for deep learning architectures, based on symmetry and invariance\cite{Geometric_DL}. It downsizes the hypothesis space, but not depends on a widely accepted axiom. The RKHS offers a rigorous mathematical foundation for classical learning methods\cite{RKHS,RKHS2,RKHS3}. But it is difficult to adaptive to deep learnings due to its limited inner product structure\cite{RKHS4}. The reproducing kernel Banach space extends from the RKHS originally designed to represent with nonlinearity, asymmetry and sparsity\cite{RKBS1,RKBS2}. Despite a undeveloped integrated model, owing to well-defined topological constructions and complete space, Banach spaces provide highly promising hypothesis space for deep learnings. Recently, the free boundary neural operator is proposed, a novel universal framework derived from the mapping between Banach spaces\cite{Banach_spaces}. \\
Here we introduce a normed space named entropy space by topology. It is defined by the collection of quotient sets composed of all the possible ways to partition a set into disjoint, non-empty subsets. Serving as an priori framework independent of network parameters, the entropy space can cover all possibilities of any deep learning model. The norming for entropy space is proved within the formal axiomatic framework. It offers the toolkit with a norm and the fundamental operations, including addition, subtraction, multiplication and Cartesian product. Based on the theory, a universal definition of depth for deep learning is given. An unified coordinate system is build on the entropy space with the state space of input as the the origin and degeneracy as the positive direction. It ranks the compression of space of information by the maximal entropy. The essence of deep learning can be summarized as minimizing information losses while compressing space quantitatively.\\


\section{Definition and fundamental operations of entropy space}\label{sec2}
Begin with the background about entropy, the Boltzmann entropy connects the microscopic states$(W)$ of a system to its macroscopic thermodynamic properties\cite{Boltzmann_entropy}. To address the need for quantification in statistical physics, the Boltzmann-Gibbs entropy (S$_{BG}$) is put forward through introducing probability\cite{Boltzmann_Gibbs_entropy}. Inspired by S$_{BG}$, Shannon CE proposed the famous information entropy (H(X))\cite{Shannon_entropy}. Then the maximum entropy principle (MaxEnt) was inferred under the principle of equal a priori probabilities hypothetically\cite{MaxEnt1}. \\
Following the spirit of priors in the Cartesian space, the foundation of entropy space theory is build on a priori framework covering all possibilities of arbitrary system.\\
\subsection{Definition of entropy space}
\begin{definition}
Let S be a system, the state space of S is set $\Omega(S)$, then we define the quotient set $\Omega_0(S)$ as the origin, based on the equivalence relation $R_0: \Omega(S)\to \Omega_0(S), \Omega_0(S)=\Omega(S)/R_0=\{[x]= \{x\}|x\in\Omega(S)\}$, denoted as $\Omega$ and $\Omega_0$.
\end{definition}
Notably, $\Omega$ represents the all possible states of arbitrary system. It is different from microstates $W$ in Boltzmann entropy which needs to meet a group of macroscale variables as constraint, such as pressure, volume and temperature. It leads to that W is always significantly smaller than $\Omega$.\\ 
As the quotient set of $\Omega$, $\Omega_0$ is a set family, and traditionally should be denoted by calligraphic letter to distinguish it from set. But the \text{\mutlutext{Ω}} is occupied by the collection of quotient sets in the following. Thus the subscripts are used to denote quotient sets in the article.\\
The physical meaning of $\Omega_0$ is taking that every state of the state space of a system is distinct, as the origin. The system can not be more complex than $\Omega_0$, unless expanding the state space. The system will be simpler than $\Omega_0$, if states form equivalence class(es). Such simplification and expansion can take $\Omega_0$ as the coordinate origin.\\
\begin{definition}
Let $\Omega$ be a set. Its collection of quotient sets $\text{\mutlutext{Ω}}(\Omega)$ is obtained under the canonical projection of the equivalence relations:  $R_i: \Omega(S)\to \Omega_i(S), \Omega_i(S)=\Omega(S)/R_i=\{[x]_{R_i}\bm{\mid} x\in\Omega\}, i\in\mathbb{N}$, which are induced by all the possible ways to partition $\Omega$ into disjoint, non-empty subsets, $\text{\mutlutext{Ω}}(\Omega)=\{ \Omega_0, ..., \Omega_i, ...\}$. The entropy space is defined as the topological space $(\Omega,\text{\mutlutext{Ω}}(\Omega))$, denoted as \text{\mutlutext{Ω}}.
\end{definition}
If $\Omega$ is a infinite set, the ways to partition is infinite too. If $\Omega$ is a finite set, then the upper bound of i is $B(|\Omega|)$. $|\Omega|$ is the cardinal number of $\Omega$. $B(|\Omega|)$ is the Bell number of $|\Omega|$. The Bell numbers are a sequence of numbers in combinatorics that count the number of ways to partition a set with elements into non-empty, disjoint subsets\cite{Bell_number}. The elements of $\Omega$ are unspecified, except $\Omega_0$. Obviously, $\Omega$ satisfies the open set axioms. One $\Omega$ maps to a unique $\text{\mutlutext{Ω}}$.\\
It should be noted that the rough set theory first quantifies the degree of simplification through equivalence classes\cite{Rough_set1,Rough_set2}. However, the entropy space is defined by collection of quotient sets, which is intrinsically different from the topological space defined by the universe and the approximation space in the rough set theory. It will lead to different compressed targets. The target of entropy space is the space of information. The target of rough set is information itself. \\
For example, in the famous LeNet-5, the typical input is a 32×32 pixel and 256 grayscale image. The set $\Omega(LeNet5)$ represents all the possible states of input, whose cardinality equals to $256^{32×32}$. The entropy space is $(\Omega(LeNet5), \text{\mutlutext{Ω}})$, denoted as $\text{\mutlutext{Ω}}$. The upper limit of equivalence class is $B(2^{256^{32×32}})$. Despite an astronomical number, as a priori framework, $\text{\mutlutext{Ω}}$ can cover all possibilities of the LeNet-5.  \\
Broadly speaking, due to the topological structure, the entropy space is mathematically more rigorous than widely used nonlinear function $f(x)=\sigma_L ( W_L\sigma_{L-1} (W_{L-1}...\sigma_1 (W_1x)...))$. Where $x$ represents the input, $f(x)$ represents the output, $L$ denotes the number of layers, $W_L$ is the weight matrix of the L-th layer, and $\sigma_L$ is the activation function of the L-th layer. The topology spaces homeomorphic to the $\text{\mutlutext{Ω}}$ can cover $f(x)$ with arbitrary precision. $\text{\mutlutext{Ω}}$ is independent of network parameters.\\
The key advance of entropy space is offering a connection between infinity to finite. Possible states of systems in the real world with precision up to the Planck scale, which are far more than  $2^{256^{32×32}}$, almost infinite. However, arbitrary output of the $f(x)$ is physically limited by the memory distributed for data, $2^{64}$ states for float64 and $2^{32}$ states for float32 per node.\\ 
For example, in the LeNet-5, if let $f(x)_n$ be output of the n-th node, and set family $\text{\mutlutext{Ω}}_n=\{ \Omega(LeNet5)/R_1, \Omega(LeNet5)/R_2, ..., \Omega(LeNet5)/R_j\}$ be composed by all the possible quotient sets of the n-th node,  then $\text{\mutlutext{Ω}}_n\subset\text{\mutlutext{Ω}}$, and $j\leq2^{32}$ due to the limit of float32. $2^{32}$ is significantly smaller than $2^{256^{32×32}}$. Similarly, for all the nodes of L-th layer, let $\text{\mutlutext{Ω}}_L$ be composed by all the quotient sets of L-th layer, then  $\text{\mutlutext{Ω}}_L\subset\text{\mutlutext{Ω}}$, , the number of elements of $\text{\mutlutext{Ω}}_L$ is significantly smaller than $2^{256^{32×32}}$ also.\\
\begin{definition}
The entropy space of real numbers $\mathbb{R}$ is written as $\text{\mutlutext{Ω}}_\mathbb{R}$ within $(\Omega(\mathbb{R}), \text{\mutlutext{Ω}}_\mathbb{R})$.
\end{definition}
Where $\Omega(\mathbb{R})$ equals to $\mathbb{R}$, so denoted for prevention of ambiguity and formal unification in a specific context.\\
\subsection{Addition and subtraction}
\begin{definition}
 Let there exist set $\Omega(x)$, $\Omega(y)\subset \Omega_\mathbb{R}$, their entropy space are $\text{\mutlutext{Ω}}_x$ and $\text{\mutlutext{Ω}}_y$. Then we defines that (1) the addition: $\text{\mutlutext{Ω}}_\mathbb{R}× \text{\mutlutext{Ω}}_\mathbb{R}\to\text{\mutlutext{Ω}}_\mathbb{R}$,  $\text{\mutlutext{Ω}}_x+ \text{\mutlutext{Ω}}_y=\textbf{\mutlutext{Ω}}(\Omega(x)+ \Omega(y))$; and (2) the subtraction: $\text{\mutlutext{Ω}}_\mathbb{R}× \text{\mutlutext{Ω}}_\mathbb{R}\to\text{\mutlutext{Ω}}_\mathbb{R}$, $\text{\mutlutext{Ω}}_x- \text{\mutlutext{Ω}}_y=\textbf{\mutlutext{Ω}}(\Omega(x)- \Omega(y))$.
\end{definition}
Where $\Omega(x)+\Omega(y)= \Omega(x)\cup \Omega(y)= \{x\bm{\mid} x\in \Omega(x)$ or $x\in\Omega(y)\}$, and $\Omega_x-\Omega_y= \{x\bm{\mid}x\in\Omega(x)$ and $x\notin\Omega(y)\}$.\\


\begin{theorem}	
	The entropy space is closed under addition and subtraction in the field of $\text{\mutlutext{Ω}}_\mathbb{R}$.
\end{theorem}
\begin{proof}
	Take any two sets $\Omega(x)$, $\Omega(y)$$\subset$$\Omega(\mathbb{R})$, and their entropy space  are $\text{\mutlutext{Ω}}_x$ and $\text{\mutlutext{Ω}}_y$. Since the $\Omega(\mathbb{R})$ are complete, $\Omega(x)+\Omega(y)$$\subset$$\Omega(\mathbb{R})$ and  $\Omega(x)-\Omega(y)$$\subset$$\Omega(\mathbb{R})$. Thus $\text{\mutlutext{Ω}}_x+ \text{\mutlutext{Ω}}_y= \text{\mutlutext{Ω}}(\Omega(x)+ \Omega(y))\subset \text{\mutlutext{Ω}}_\mathbb{R}$. And $\text{\mutlutext{Ω}}_x- \text{\mutlutext{Ω}}_y= \text{\mutlutext{Ω}}(\Omega(x)- \Omega(y))\subset \text{\mutlutext{Ω}}_\mathbb{R}$.
\end{proof}


\begin{definition}
The entropy space of empty set $\Omega(\emptyset)= \emptyset$ is written as $\text{\mutlutext{Ω}}_\emptyset$ within the $(\Omega(\emptyset), \text{\mutlutext{Ω}}_\emptyset)$.
\end{definition}


\begin{theorem}	
	The addition operation of entropy space has (1)one identity element $\text{\mutlutext{Ω}}_\emptyset$, (2)inverse elements, and satisfies (3)the commutative law (4)the associative law.
\end{theorem}
\begin{proof}
	Take any any three sets $\Omega(x)$, $\Omega(y)$, $\Omega(z)$$\subset$$\Omega_\mathbb{R}$, and their entropy spaces  $\text{\mutlutext{Ω}}_x$, $\text{\mutlutext{Ω}}_y$, $\text{\mutlutext{Ω}}_z\subset\text{\mutlutext{Ω}}_\mathbb{R}$. (1)One identity element: since $\text{\mutlutext{Ω}}_x+\text{\mutlutext{Ω}}_\emptyset=\text{\mutlutext{Ω}}(\Omega(x)+\emptyset)$, $\text{\mutlutext{Ω}}_\emptyset$ is an identity element. Assume there exists another entropy space $\text{\mutlutext{Ω}}_e$ that is also an identity element. In particular, take $\text{\mutlutext{Ω}}_x=\text{\mutlutext{Ω}}_\emptyset$, then $\text{\mutlutext{Ω}}_\emptyset+\text{\mutlutext{Ω}}_e=\text{\mutlutext{Ω}}_\emptyset$. But by the definition of addition, $\text{\mutlutext{Ω}}_e+\text{\mutlutext{Ω}}_\emptyset=\text{\mutlutext{Ω}}_e$. Hence, $\text{\mutlutext{Ω}}_e=\text{\mutlutext{Ω}}_\emptyset$, $\text{\mutlutext{Ω}}_\emptyset$ is the unique identity element. (2)Inverse elements: $\text{\mutlutext{Ω}}_x-\text{\mutlutext{Ω}}_x=\text{\mutlutext{Ω}}(\Omega(x)-\Omega(x))=\text{\mutlutext{Ω}}_\emptyset$. (3)The commutative law: $\text{\mutlutext{Ω}}_x+\text{\mutlutext{Ω}}_y=\text{\mutlutext{Ω}}(\Omega(x)+\Omega(y))=\text{\mutlutext{Ω}}(\Omega(y)+\Omega(x))=\text{\mutlutext{Ω}}_y+\text{\mutlutext{Ω}}_x$. (4)The associative law: $\text{\mutlutext{Ω}}_x+(\text{\mutlutext{Ω}}_y+\text{\mutlutext{Ω}}_z)=\text{\mutlutext{Ω}}_x+\text{\mutlutext{Ω}}(\Omega(y)+\Omega(z))=\text{\mutlutext{Ω}}(\Omega(x)+\Omega(y)+\Omega(z))=\text{\mutlutext{Ω}}(\Omega(x)+\Omega(y))+\text{\mutlutext{Ω}}_z=(\text{\mutlutext{Ω}}_x+\text{\mutlutext{Ω}}_y)+\text{\mutlutext{Ω}}_z$.
\end{proof}
A summary is provided in Tab. 1. Follow them there is a unique  $\text{\mutlutext{Ω}}_x+\text{\mutlutext{Ω}}_y$$\subset\text{\mutlutext{Ω}}_\mathbb{R}$. \\ 
\begin{table}[htbp]
	\centering
	\caption{Laws of addition of entropy space}
	\label{tab:laws_addition_es}
	
	\renewcommand{\arraystretch}{1.0}
	
	\begin{tabular}{
			>{\centering\arraybackslash}p{0.30\textwidth}|
			>{\centering\arraybackslash}p{0.20\textwidth}|
			>{\centering\arraybackslash}p{0.40\textwidth}
		}
		\hline
		
		\textbf{Identity element}
		&
		\multirow{4}{*}{
			$\displaystyle
			\begin{gathered}
				\forall\,\text{\mutlutext{Ω}}_x,\,
				\text{\mutlutext{Ω}}_y,\,
				\text{\mutlutext{Ω}}_z
				\subset\text{\mutlutext{Ω}}_{\mathbb{R}}
			\end{gathered}$
		}
		&
		$\text{\mutlutext{Ω}}_\emptyset,\ \mathrm{s.t.},\ 
		\text{\mutlutext{Ω}}_x+\text{\mutlutext{Ω}}_\emptyset
		=
		\text{\mutlutext{Ω}}_\emptyset+\text{\mutlutext{Ω}}_x
		=
		\text{\mutlutext{Ω}}_x$
		\\
		
		\hhline{|-~|-}
		
		\textbf{Inverse element}
		&
		&
		$\mathrm{s.t.},\ 
		\text{\mutlutext{Ω}}_x-\text{\mutlutext{Ω}}_x
		=
		\text{\mutlutext{Ω}}_\emptyset$
		\\
		
		\hhline{|-~|-}
		
		\textbf{Commutative law}
		&
		&
		$\mathrm{s.t.},\ 
		\text{\mutlutext{Ω}}_x+\text{\mutlutext{Ω}}_y
		=
		\text{\mutlutext{Ω}}_y+\text{\mutlutext{Ω}}_x$
		\\
		
		\hhline{|-~|-}
		
		\textbf{Associative law}
		&
		&
		$\mathrm{s.t.},\ 
		\text{\mutlutext{Ω}}_x+
		(\text{\mutlutext{Ω}}_y+\text{\mutlutext{Ω}}_z)
		=
		(\text{\mutlutext{Ω}}_x+\text{\mutlutext{Ω}}_y)+\text{\mutlutext{Ω}}_z$
		\\
		
		\hline
	\end{tabular}
\end{table}

\subsection{Multiplication}

 
\begin{definition}
Let there exist $a \in\mathbb{R}$, $\Omega\subset\Omega(\mathbb{R})$, the entropy space is $(\Omega(x),\text{\mutlutext{Ω}}_x)$. Then we defines the multiplication: $\mathbb{R}×\text{\mutlutext{Ω}}_\mathbb{R}\to\text{\mutlutext{Ω}}_\mathbb{R}$, $a\cdot\text{\mutlutext{Ω}}_x=\text{\mutlutext{Ω}}(a\cdot\Omega(x))$,  also denoted as $a\text{\mutlutext{Ω}}_x$.
\end{definition}


\begin{theorem}	
	The entropy space is closed under multiplication in the field of $\text{\mutlutext{Ω}}_\mathbb{R}$.
\end{theorem}
\begin{proof}
Take any $a \in\mathbb{R}$, $\Omega(x)\subset\Omega_\mathbb{R}$, and the entropy space of $\Omega(x)$ is $\text{\mutlutext{Ω}}_x$. Since the $\Omega_\mathbb{R}$ are complete, $a\Omega(x)\subset\Omega_\mathbb{R}$. Thus  $a\text{\mutlutext{Ω}}_x=\text{\mutlutext{Ω}}(a\Omega(x))\subset\text{\mutlutext{Ω}}_\mathbb{R}$.
\end{proof}


\begin{theorem}	
	The multiplication of entropy space has (1)one identity element 1, satisfies (2)the distributive law over entropy space addition, (3)the distributive law over scalar addition and (4)the associative law.
\end{theorem}
\begin{proof}
Take any $a$, $b$ $\in\mathbb{R}$, and any two sets $\Omega(x)$, $\Omega(y)\subset\Omega_\mathbb{R}$, their entropy spaces $\text{\mutlutext{Ω}}_x$, $\text{\mutlutext{Ω}}_y\subset\text{\mutlutext{Ω}}_\mathbb{R}$. (1)One identity element 1: since $1\cdot\text{\mutlutext{Ω}}_x=\text{\mutlutext{Ω}}(1\cdot\Omega(x))=\text{\mutlutext{Ω}}_x$, 1 is an identity element. Assume there exists another entropy space $e\subset\mathbb{R}$ that is also an identity element. $e\cdot\text{\mutlutext{Ω}}_x=\text{\mutlutext{Ω}}(e\cdot\Omega(x))=\text{\mutlutext{Ω}}(1\cdot\Omega(x))$. Hence, $e= 1$, 1 is the unique identity element.  (2)The distributive law over entropy space addition: $a\cdot(\text{\mutlutext{Ω}}_x+\text{\mutlutext{Ω}}_y)=\text{\text{\mutlutext{Ω}}}(a\cdot\Omega(x)+a\cdot\Omega(y))=a\cdot\text{\mutlutext{Ω}}_x+a\cdot\text{\mutlutext{Ω}}_y$. (3)The distributive law over scalar addition: $(a+b)\cdot\text{\mutlutext{Ω}}_x=\text{\mutlutext{Ω}}(a\cdot\Omega(x)+b\cdot\Omega(y))=a\cdot\text{\mutlutext{Ω}}_x+b\cdot\text{\mutlutext{Ω}}_x$. (4)The associative law: $a\cdot(b\cdot\text{\mutlutext{Ω}}_x)=\text{\mutlutext{Ω}}(a\cdot b\cdot \Omega(x))=(a\cdot b)\cdot\text{\mutlutext{Ω}}_x$.
\end{proof}
 A summary is provided in Tab. 2. Follow them, there is a unique $a\text{\mutlutext{Ω}}_x\subset\text{\mutlutext{Ω}}_\mathbb{R}$.\\


\begin{table}[htbp]
	\centering
	\caption{Laws of multiplication of entropy space}
	\label{tab:laws_scalar_multiplication}
	
	\renewcommand{\arraystretch}{1.6}
	
	\begin{tabular}{
			>{\centering\arraybackslash}p{4cm}|
			>{\centering\arraybackslash}p{2cm}|
			>{\centering\arraybackslash}p{5.4cm}
		}
		\hline
		
		\textbf{Identity element}
		&
		\multirow{4}{*}{
			$\displaystyle
			\begin{gathered}
				\forall\,a,b\in\mathbb{R},\\[10pt]
				\forall\,\text{\mutlutext{Ω}}_x,\,\text{\mutlutext{Ω}}_y
				\subset\text{\mutlutext{Ω}}_{\mathbb{R}}
			\end{gathered}$
		}
		&
		$1,\ \mathrm{s.t.}, 1\cdot\text{\mutlutext{Ω}}_x=\text{\mutlutext{Ω}}_x$
		\\
		
		\hhline{|-~|-}
		
		\textbf{Distributive law over entropy space addition}
		&
		&
		$\mathrm{s.t.}, 
		a\cdot(\text{\mutlutext{Ω}}_x+\text{\mutlutext{Ω}}_y)
		=
		a\cdot\text{\mutlutext{Ω}}_x+a\cdot\text{\mutlutext{Ω}}_y$
		\\
		
		\hhline{|-~|-}
		
		\textbf{Distributive law over scalar addition}
		&
		&
		$\mathrm{s.t.}, 
		(a+b)\cdot\text{\mutlutext{Ω}}_x
		=
		a\cdot\text{\mutlutext{Ω}}_x+b\cdot\text{\mutlutext{Ω}}_x$
		\\
		
		\hhline{|-~|-}
		
		\textbf{Associative law}
		&
		&
		$\mathrm{s.t.}, 
		a\cdot(b\cdot\text{\mutlutext{Ω}}_x)
		=
		(a\cdot b)\cdot\text{\mutlutext{Ω}}_x$
		\\
		
		\hline
	\end{tabular}
\end{table}

\subsection{Cartesian product}


\begin{definition}
Let there exist n sets $\Omega(X_1),..., \Omega(X_j),..., \Omega(X_n)\subset\Omega(\mathbb{R})$, and the entropy space $(\Omega(X_1),\text{\mutlutext{Ω}}_1),...,(\Omega(X_j),\text{\mutlutext{Ω}}_j),...,(\Omega(X_n),\text{\mutlutext{Ω}}_n)$. Then we defines the Cartesian product: $\prod^n_{j=1}\text{\mutlutext{Ω}}_j=\text{\mutlutext{Ω}}_1× ...\text{\mutlutext{Ω}}_j× ...\text{\mutlutext{Ω}}_n=\text{\mutlutext{Ω}}(\Omega(X_1)× ...\Omega(X_j)× ...\Omega(X_n))$, denoted as $\text{\mutlutext{Ω}}^n$. $\text{\mutlutext{Ω}}_j$ is called the j-th coordinate of $\text{\mutlutext{Ω}}^n$.
\end{definition}


\begin{definition}
$\text{\mutlutext{Ω}}^n_\mathbb{R}$ is the n-dimensional entropy space
\end{definition}


\section{Norming of entropy space}\label{sec3}

\subsection{Metric space}

\begin{theorem}	
	Let there exist entropy space $(\Omega,\text{\mutlutext{Ω}})$, $\Omega\subset \Omega(\mathbb{R})$, set $\Omega_x$,  $\Omega_y\in\text{\mutlutext{Ω}}$, then $(\text{\mutlutext{Ω}},\rho)$ is a metric space about distance $\rho$: $\text{\mutlutext{Ω}}×\text{\mutlutext{Ω}}\to\mathbb{N}$, $\rho(\Omega_x, \Omega_y)=\vert\Omega_x\triangle\Omega_y\vert$.
\end{theorem}
Where $\Omega_x\triangle\Omega_y=(\Omega_x-\Omega_y)\cup(\Omega_y-\Omega_x)=\{x\vert\in\Omega_x\cup\Omega_y, x\notin\Omega_x\cap\Omega_y\}$. $\vert\Omega_x\triangle\Omega_y\vert$ is the cardinality. The elements of $\Omega_x$ or $\Omega_y$ are equivalence classes.\\
\begin{proof}
Let there exist a entropy space $(\Omega,\text{\mutlutext{Ω}})$, $\Omega\subset\Omega(\mathbb{R})$. Take any $\Omega_x,\Omega_y,\Omega_z\in\text{\mutlutext{Ω}}$. (1)Non-negativity and positive definiteness: according to the definition, $\rho(\Omega_x,\Omega_y)\geq0$. If $\Omega_x=\Omega_y$, $\Omega_x\triangle\Omega_y=\emptyset$, $\rho(\Omega_x,\Omega_y)$. Due to $\Omega_x=\Omega_0 /R_x $, $\Omega_y=\Omega_0 /R_y$ by definition 2, if $\Omega_y\neq\Omega_y$, there are at least two different elements between $\Omega_x$ and $\Omega_y$. Hence if $\Omega_x\neq\Omega_y$, $\rho(\Omega_x, \Omega_y)\geq2>0$. (2)Symmetry: obviously $\rho(\Omega_x, \Omega_y)=\rho(\Omega_y, \Omega_x)$ by the definition. (3)Triangle inequality: $\rho(\Omega_x, \Omega_y)\leq\rho(\Omega_x, \Omega_z)+\rho(\Omega_z, \Omega_y)$. The Venn diagram of $\Omega_x$, $\Omega_y$ and $\Omega_z$ can partition into 3 distinct regions as shown in Tab. 3. Yellow area represents $\rho(\Omega_x, \Omega_y)$, and blue area represents $\rho(\Omega_x, \Omega_z)+\rho(\Omega_z, \Omega_y)$. Every possible case enumerated satisfies the condition. Therefore, we conclude that $(\text{\mutlutext{Ω}},\rho)$ is a metric space about $\rho: \text{\mutlutext{Ω}}×\text{\mutlutext{Ω}}\to\mathbb{N}$.
\end{proof}

\begin{table}[htbp]
	\centering
	\caption{Enumeration of all the possible cases of triangle inequality}
	\label{tab:venn_cases}
	
	\renewcommand{\arraystretch}{1}
	
	\begin{tabular}{
			@{}>{\centering\arraybackslash}p{3cm}|
			>{\centering\arraybackslash}p{2cm}|
			>{\centering\arraybackslash}p{2cm}|
			>{\centering\arraybackslash}m{4cm}@{}
		}
		
		\hline
		
		\textbf{Distinct regions}
		&
		$\boldsymbol{\rho(\Omega_x,\Omega_y)}$
		&
		$\begin{gathered}
			\boldsymbol{\rho(\Omega_x,\Omega_z)+}\\
			\boldsymbol{\rho(\Omega_z,\Omega_y)}
		\end{gathered}$
		&
		\textbf{
			Whether
			$\rho(\Omega_x,\Omega_y)
			\leq
			\rho(\Omega_x,\Omega_z)+
			\rho(\Omega_z,\Omega_y)$
		}
		\\
		
		\hline
		
		\multirow{4}{*}{
			\begin{minipage}{0.22\textwidth}
				\vspace{-1.5cm}
				\centering
			
				\ding{172} $\Omega_z\subseteq\Omega_x\cup\Omega_y,$
				
				and
				
				$(\Omega_x\subseteq\Omega_y
				\text{ or }
				\Omega_y\subseteq\Omega_x)$
			\end{minipage}
		}
		&
		\includegraphics[
		width=0.73\linewidth,
		trim=0 0 0 0,
		clip
		]{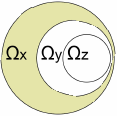}
		&
		\includegraphics[
		width=0.75\linewidth,
		trim=0 0 0 0,
		clip
		]{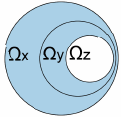}
		&
		\begin{minipage}{0.26\textwidth}
			\vspace{-1.5cm}
			\centering
			Yes, and yes when
			$(\Omega_x,\Omega_y)
			\rightarrow
			(\Omega_y,\Omega_x)$.
		\end{minipage}
		\\
		
		\hhline{|~|-|-|-}
		
		&
		\includegraphics[
		width=0.73\linewidth,
		trim=0 0 0 0,
		clip
		]{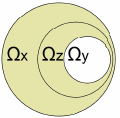}
		&
		\includegraphics[
		width=0.75\linewidth,
		trim=0 0 0 0,
		clip
		]{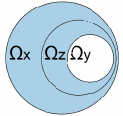}
		&
		\begin{minipage}{0.26\textwidth}
			\vspace{-1.5cm}
			\centering
			Yes, and yes when
			$(\Omega_x,\Omega_y)
			\rightarrow
			(\Omega_y,\Omega_x)$.
		\end{minipage}
		\\
		
		\hhline{|~|-|-|-}
		
		&
		\includegraphics[
		width=0.75\linewidth,
		trim=0 0 0 0,
		clip
		]{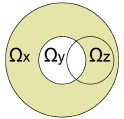}
		&
		\includegraphics[
		width=0.75\linewidth,
		trim=0 0 0 0,
		clip
		]{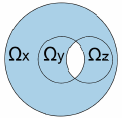}
		&
		\begin{minipage}{0.26\textwidth}
			\vspace{-1.5cm}
			\centering
			Yes, and yes when
			$(\Omega_x,\Omega_y)
			\rightarrow
			(\Omega_y,\Omega_x)$.
		\end{minipage}
		\\
		
		\hhline{|~|-|-|-}
		&
		\includegraphics[
		width=0.75\linewidth,
		trim=0 0 0 0,
		clip
		]{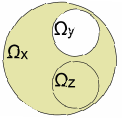}
		&
		\includegraphics[
		width=0.75\linewidth,
		trim=0 0 0 0,
		clip
		]{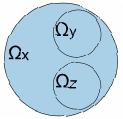}
		&
		\begin{minipage}{0.26\textwidth}
			\vspace{-1.5cm}
			\centering
			Yes, and yes when
			$(\Omega_x,\Omega_y)
			\rightarrow
			(\Omega_y,\Omega_x)$.
		\end{minipage}
		\\
		
		\hhline{|-|-|-|-}
		
		\multirow{4}{*}{
			\begin{minipage}{0.22\textwidth}
				\vspace{-1.5cm}
				\centering

				\ding{173} $\Omega_z\subseteq\Omega_x\cup\Omega_y,$
				
				and
				
				$(\Omega_x\nsubseteq\Omega_y
				\text{ and }
				\Omega_y\nsubseteq\Omega_x)$
			\end{minipage}
		}
		&
		\makebox[\linewidth][c]{%
			\includegraphics[
			width=1\linewidth,
			height=0.205\textheight,
			keepaspectratio
			]{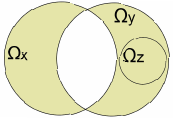}%
		}
		&
		\makebox[\linewidth][c]{%
			\includegraphics[
			width=1\linewidth,
			height=0.205\textheight,
			keepaspectratio
			]{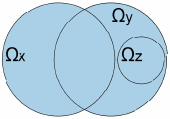}%
		}
		&
		\begin{minipage}{0.36\textwidth}
			\vspace{-1.5cm}
			\centering
			Yes, and yes when
			$(\Omega_x,\Omega_y)
			\rightarrow
			(\Omega_y,\Omega_x)$.
			
			Also yes when
			$\Omega_x\cap\Omega_y=\varnothing$.
		\end{minipage}
		
		\\
		
		\hhline{|~|-|-|-}
		
		
		&
		\makebox[\linewidth][c]{%
			\includegraphics[
			width=1\linewidth,
			height=0.205\textheight,
			keepaspectratio
			]{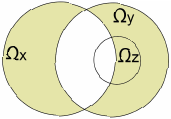}%
		}
		&
		\makebox[\linewidth][c]{%
			\includegraphics[
			width=1\linewidth,
			height=0.205\textheight,
			keepaspectratio
			]{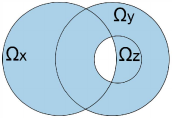}%
		}
		&
		\begin{minipage}{0.26\textwidth}
			\vspace{-1.5cm}
			\centering
			Yes, and yes when
			$(\Omega_x,\Omega_y)
			\rightarrow
			(\Omega_y,\Omega_x)$.
		\end{minipage}
		
		\\
		
		\hhline{|~|-|-|-}
		
		
		&
		\makebox[\linewidth][c]{%
			\includegraphics[
			width=1\linewidth,
			height=0.205\textheight,
			keepaspectratio
			]{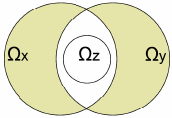}%
		}
		&
		\makebox[\linewidth][c]{%
			\includegraphics[
			width=1\linewidth,
			height=0.205\textheight,
			keepaspectratio
			]{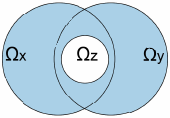}%
		}
		&
		\begin{minipage}{0.26\textwidth}
			\vspace{-1.5cm}
			\centering
			Yes.
		\end{minipage}
		\\
		
		\hhline{|~|-|-|-}
		
		
		&
		\makebox[\linewidth][c]{%
			\includegraphics[
			width=1\linewidth,
			height=0.205\textheight,
			keepaspectratio
			]{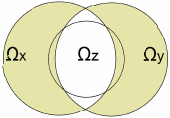}%
		}
		&
		\makebox[\linewidth][c]{%
			\includegraphics[
			width=1\linewidth,
			height=0.205\textheight,
			keepaspectratio
			]{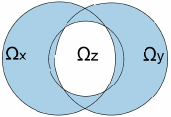}%
		}
		&
		\begin{minipage}{0.26\textwidth}
			\vspace{-1.5cm}
			\centering
			Yes.
		\end{minipage}
		\\
		
		\hline
		
		
		\begin{minipage}{0.22\textwidth}
			\vspace{-1.5cm}
			\centering

			\ding{174} $\Omega_z\nsubseteq\Omega_x\cup\Omega_y$
		\end{minipage}
		&
		\makebox[\linewidth][c]{%
			\includegraphics[
			width=1\linewidth,
			height=0.205\textheight,
			keepaspectratio
			]{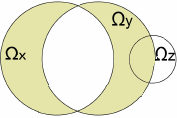}%
		}
		&
		\makebox[\linewidth][c]{%
			\includegraphics[
			width=1\linewidth,
			height=0.205\textheight,
			keepaspectratio
			]{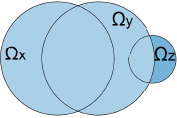}%
		}
		&
		\begin{minipage}{0.31\textwidth}
			\vspace{-0.20cm}
			\centering
			Yes. The
			$\rho(\Omega_x,\Omega_y)$
			(yellow area) is the same in \ding{172}  or \ding{173} ,
			while the
			$\rho(\Omega_x,\Omega_z)
			+
			\rho(\Omega_z,\Omega_y)$
			(blue area) in \ding{172}  or \ding{173}  additionally
			includes the
			$\Omega_z-(\Omega_x\cup\Omega_y)$
			area (deeper blue).
		\end{minipage}
		
		\\
		
		\hline
		
	\end{tabular}
\end{table}

\subsection{Vector space}

\begin{theorem}
	
	Let there exist set $\Omega\subset\Omega(\mathbb{R})$, and $(\Omega, \text{\mutlutext{Ω}})$, then the entropy space Ω is a vector space. 

\end{theorem}
\begin{proof}
	(1)$\text{\mutlutext{Ω}}$ is closed under addition in the field of $\text{\mutlutext{Ω}}_\mathbb{R}$, has one identity element $\text{\mutlutext{Ω}}_\emptyset$
    Let there exist set and inverse elements, satisfies the commutative law and the associative law, as proved in the theorem 1.5 and 1.7. (2)$\text{\mutlutext{Ω}}$ is closed under multiplication in the field of $\text{\mutlutext{Ω}}_\mathbb{R}$, has one identity element 1, satisfies the distributive law over entropy space addition, the distributive law over scalar addition, and the associative law, as proved in theorem 1.9 and 1.10. Therefore $\text{\mutlutext{Ω}}$ is a vector space.
    
\end{proof}

\subsection{Normed space}
\begin{theorem}
Let there exist set $\Omega\subset\Omega(\mathbb{R})$, $(\Omega, \text{\mutlutext{Ω}})$, and $\Omega_x,\Omega_y\in\text{\mutlutext{Ω}}$, then the entropy space $\text{\mutlutext{Ω}}$ is a normed space about the distance $\rho:\text{\mutlutext{Ω}}×\text{\mutlutext{Ω}}\to\mathbb{N}$, $\rho(\Omega_x,\Omega_y)=\vert\Omega_x\triangle\Omega_y\vert$, denoted as $(\text{\mutlutext{Ω}},\rho)$.
\end{theorem}
\begin{proof}
	(1)$\text{\mutlutext{Ω}}$ is vector space as proved in theorem 2.2. (2)$\rho$ is a metric of $(\text{\mutlutext{Ω}},\rho)$ as proved in theorem 2.1. Therefore $(\text{\mutlutext{Ω}},\rho)$ is a normed space about $\rho$. 
\end{proof}

\section{Understanding of deep learning with entropy space}\label{sec4}

\subsection{Shannon-Boltzmann isoentropy}

\begin{definition}
Let there exist set $\Omega\subset\Omega(\mathbb{R})$, entropy space $(\Omega, \text{\mutlutext{Ω}})$, $\Omega_x\in\text{\mutlutext{Ω}}$, then $H_m(\Omega_x)$ is called the maximum entropy of $\Omega_x$, $H_m(\Omega_x): \text{\mutlutext{Ω}}\to \mathbb{R}_+$, $H_m(\Omega_x)=\{\log\vert\Omega_x\vert, \Omega_x\neq\emptyset;0,\Omega_x=\emptyset\}$.
\end{definition}

$H_m(\Omega_x)$ is the maximal value of information entropy for a system. It is obtained by the MaxEnt, when probability of any state equals to $1/\vert\Omega_x\vert$. $H_m(\Omega_x)$ is always greater than or equal to zero. The x represents a variable symbol of quotient topology. Since the empty set represents a completely determinate state with no possibilities whatsoever, by convention, $H_m(\emptyset)=0$.\\
Take a six-sided die as an example, $\Omega(die)=\{1,2,3,4,5,6\}$. Regardless of whether the die is fair, $H_m(\Omega_0)=\log6$, $\vert\text{\mutlutext{Ω}}\vert=B(6)=203$. $B(6)$ is the Bell number with 6 elements to partition. From the 203 situations, if we pick the quotient topology called big/small, denoted as bs: $\Omega\to\Omega_{bs}$, $\Omega_{bs}=\{[1]_{small}=\{1,2,3\}, [6]_{big}=\{4,5,6\}\}$, then $H_m(\Omega_{bs})=\log2$. 

\begin{definition}
Let there exist set $\Omega\subset\Omega(\mathbb{R})$, entropy space $(\Omega, \text{\mutlutext{Ω}})$, $\Omega_a\in\text{\mutlutext{Ω}}$, then the collection of quotient sets with equal maximum entropy $\text{\mutlutext{Ω}}_{SB}(\Omega_a)=\{\Omega_x\bm{\mid} H_m(\Omega_x)=H_m(\Omega_a)\}$, is defined as Shannon-Boltzmann isoentropy.

\end{definition}

$\text{\mutlutext{Ω}}_{SB}$ is named this way because of Shannon entropy and Boltzmann entropy are equal when the prior probabilities of all system states are equal (MaxEnt). It constitutes a bridge between physical world and digital world.\\
Take the perceptron as a example, the standard perceptron is modeled as $y=f(\sum (w_i x_i)+b)$, where $x_i$ is the input, y is the output, $f$ is the activation function, $w_i$ is the weight of $x_i$, and $b$ is the bias. The $\Omega_0(x_i)$ can be arbitrarily large, also for the $H_m(\Omega_0(x_i))$. However, the  $H_m(\Omega_0(y))$ is limited by the memory allocated for $y$. If float32, then all the possible states of $\Omega_0(y)$ forms a isoentropy $\text{\mutlutext{Ω}}_{SB}(\Omega_0(y))$, $H_m(\Omega_0(y))=\log2^{32}$. It means that, although the value of y is uncertain, the maximum amount of information is certain and equals to $H_m(\Omega_0(y))$. Usually $H_m(\Omega_0(x_i))\gg H_m(\Omega_0(y))$. Hence the essence of the perceptron is compression of the space of information.

\subsection{Reinterpret of deep learning}

\begin{definition}
Let there exist set $\Omega\subset\Omega(\mathbb{R})$, entropy space $(\Omega, \text{\mutlutext{Ω}})$, $\Omega_x\in\text{\mutlutext{Ω}}$, then depth is defined as the function $D(\Delta H_m): \text{\mutlutext{Ω}} \to \mathbb{R}$, that positively correlated with $\Delta H_m=H_m(\Omega_0)-H_m(\Omega_x)$.
\end{definition}

The concept of depth previously lacked a clear definition, people often described it by number of layers. However such kind of descriptions encounter problems like “how many layers qualifies as deep? ". $D(\Delta H_m)$ is characterized by both physical relevance and mathematical precision.\\

\begin{definition}
A system whose depth $D(\Delta H_m)$ markedly changes during computation process is called deep learning.
\end{definition}

Here we offer a novel definition of DL more general than traditional ones. It can include human mind and distinguish between the others non-DL.\\

\begin{definition}
Let there exist a system of deep learning with n-dimensional input. The state space of input is $\Omega\subset \Omega^n(\mathbb{R})$, and the entropy space is $(\Omega,\text{\mutlutext{Ω}})$. Let there exist $\Omega_x,\Omega_y \in \text{\mutlutext{Ω}}$ and exist the mapping: $\text{\mutlutext{Ω}} \to \text{\mutlutext{Ω}}, \Omega_x\to \Omega_y$. Then if $H_m(\Omega_x)>H_m(\Omega_y)$, we call it degeneracy. If $H_m(\Omega_x)<H_m(\Omega_y)$, we call it regeneration. If $H_m(\Omega_x)=H_m(\Omega_y)$, we call it co-isoentropy.
\end{definition} 

Resolution of input of any system subjectively in a agent perspective such as a camera, an AI program or a human-being, should be discrete and finite, not continuous or infinite. However state space of input can be a astronomical number, much less the entropy space. It will require massive computing resources, if $D(\Delta H_m)$ does not change during solving. Definition 13 specifies the directions when $D(\Delta H_m)$ changes. Since the relevant concepts in topology are too rigorous, the definition of degeneracy is choose from physics.\\
The word degeneracy originates from the German word Entartung, is widely used in various branches of physics. Its meaning refers to the phenomenon where multiple states that were originally distinct degenerate, merge, or converge into the same characteristic or a coarser level due to certain conditions. Regeneration is the reverse process of degeneracy.\\

\begin{definition}
Let there exist deep learning system whose input is within the state space $\Omega\subset \Omega^n(\mathbb{R})$, entropy space $(\Omega,\text{\mutlutext{Ω}})$, $\Omega_x \in \text{\mutlutext{Ω}}$. Then the entropy space based coordinate system of deep learning is defined as: (1)The basis depends on the input $\Omega$. (2)The norm depends on the $\rho(\Omega_0,\Omega_x)$. (3)The origin is $\Omega_0$. (4)The positive direction is the degeneracy direction.

\end{definition} 

\begin{figure}
    \centering
    \includegraphics[width=0.8\linewidth]{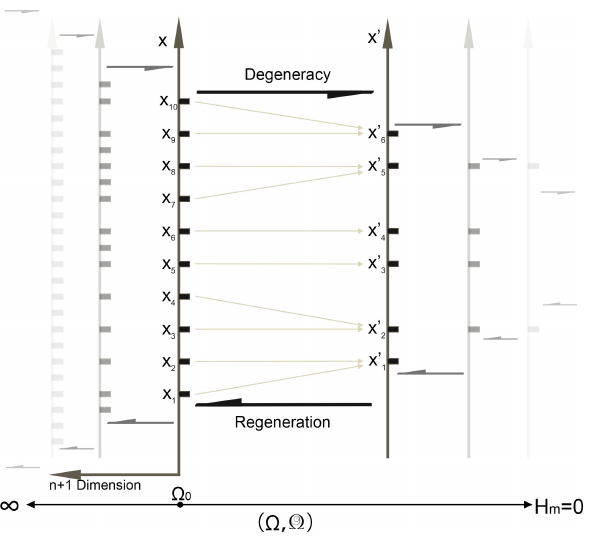}
    \caption{Schematic of entropy space based coordinate system}
    \label{fig:placeholder}
\end{figure}

Where $\Omega$ can be in arbitrary dimension, if input is a n-dimensional subspace of a vector space ($\Omega$ could be any set, contain vector space). Although the mapping $\Omega_x \to 
\Omega_{x^\prime}$ shown in Fig. 1 is drawn in one dimension, to make it easier both for readers to understand and authors to draw, actually can be in arbitrary dimension based on $\Omega$. The endpoint of degeneracy is $H_m=0$. $(\Omega, \text{\mutlutext{Ω}})$ can be expanded endlessly by the fundamental operations of entropy space. For example, the all possible states of a multimodal deep learning model can be be represented as $\text{\mutlutext{Ω}}(multimodal)=\text{\mutlutext{Ω}}(text)×\text{\mutlutext{Ω}}(image)×\text{\mutlutext{Ω}}(audio)×\text{\mutlutext{Ω}}(video)×\text{\mutlutext{Ω}}(sensor)$.\\
Based on entropy space theory, deep learning could be understood by very simple description. In a deep learning network, from input $\Omega$ to any node x, $\Omega_x=\Omega_0/R_x\in \text{\mutlutext{Ω}}$. Then the topology spaces homeomorphic to the $(\Omega,\Omega_x)$ can cover all the possibilities of node x. Due to the $H_m(\Omega_x)$ is limited by storage resolution, it makes up a $\text{\mutlutext{Ω}}_{SB}(\Omega_x)$ onto the node x. The space is quantitatively compressed by designed network architecture, then the information is compressed when getting through. Network parameters decide which part of space should be compressed, which are optimized automatically through gradient descent of loss function and backpropagation. Most of space is gradually compressed to the one state of zero by activation functions, to make room for features. \\

\section{Conclusion}\label{sec5}
Compared with previous studies, the main contribution of this work lies in a novel normed space, the entropy space, proposed for unified mathematical framework of deep learning models. It offers a brief but compelling explanation of why deep learning can automatically learn features. Analogous to the spatial rectangular coordinate system, the coordinate system based on the entropy space theory has the potential to serve as the mathematical foundation for cross-architecture learning, and even for artificial general intelligence.\\

\bibliography{sn-bibliography}

\end{document}